\documentclass[10pt,a4paper]{article}
\usepackage[utf8]{inputenc}
\usepackage{amsmath}
\usepackage{amsfonts}
\usepackage{amssymb}
\usepackage{hyperref}
\usepackage{graphicx}
\usepackage{xcolor}

\usepackage{subfigure}
\usepackage{amsthm}

\DeclareMathOperator*{\argmin}{arg\,min}

\newtheorem{theorem}{Theorem}

\title{An End-to-End Graph Convolutional Kernel Support Vector Machine}
\author{Padraig Corcoran \\ School of Computer Science \& Informatics, Cardiff University}

\date{}
\begin{document}
\maketitle

\begin{abstract}
A novel kernel-based support vector machine (SVM) for graph classification is proposed. The SVM feature space mapping consists of a sequence of graph convolutional layers, which generates a vector space representation for each vertex, followed by a pooling layer which generates a reproducing kernel Hilbert space (RKHS) representation for the graph. The use of a RKHS offers the ability to implicitly operate in this space using a kernel function without the computational complexity of explicitly mapping into it. The proposed model is trained in a supervised end-to-end manner whereby the convolutional layers, the kernel function and SVM parameters are jointly optimized with respect to a regularized classification loss. This approach is distinct from existing kernel-based graph classification models which instead either use feature engineering or unsupervised learning to define the kernel function. Experimental results demonstrate that the proposed model outperforms existing deep learning baseline models on a number of datasets.
\end{abstract}

\section{Introduction}
The world contains much implicit structure which can be modelled using a graph. For example, an image can be modelled as a graph where objects (e.g. person, chair) are modelled as vertices and their pairwise relationships (e.g. sitting) are modelled as edges \cite{krishna2017visual}. This representation has led to useful solutions for many vision problems including image captioning and visual question answering \cite{chen2019scene}. Similarly, a street network can be modelled as a graph where locations are modelled as vertices and street segments are modelled as edges. This representation has led to useful solutions for many transportation problems including the placement of electrical vehicle charging stations \cite{gagarin2018multiple}.

Given the ubiquity of problems which can be modelled in terms of graphs, performing machine learning on graphs represents an area of great research interest. Advances in the application of deep learning or neural networks to sequence spaces in the context of natural language processing and fixed dimensional vector spaces in the context of computer vision has led to much interest in applying deep learning to graphs. There exist many types of machine learning tasks one may wish to perform on graphs. These include vertex classification, graph classification, graph generation \cite{you2018graphrnn} and learning implicit/hidden structures \cite{franceschi2019learning}. In this work we focus on the task of graph classification. Examples of graph classification tasks include human activity recognition where human pose is modelled using a skeleton graph \cite{yan2018spatial}, visual scene understanding where the scene is modelled using a \textit{scene graph} \cite{xu2017scenegraph} and semantic segmentation of three dimensional point clouds where the point cloud is modelled as a graph of geometrically homogeneous elements \cite{Landrieu_2018_CVPR}.

Graph convolutional is the most commonly used deep learning architecture applied to graphs. This architecture consists of a sequence of convolutional layers where each layer iteratively updates a vector space representation of each vertex. In their seminal work, Gilmer et al. \cite{gilmer2017} demonstrated that many different convolutional layers can be formulated in terms of a framework containing two steps. In the first step, message passing is performed where each vertex receives messages from adjacent vertices regarding their current representation. In the second step, each vertex performs an update of its representation which is a function of its current representation and the messages it received in the previous step. In order to perform graph classification given a sequence of convolutional layers, the set of vertex representations output from this sequence must be integrated to form a graph representation. This graph representation can subsequently be used to predict a corresponding class label. We refer to this task of integrating vertex representations as \textit{vertex pooling} and it represents the focus of this article. Note that, Gilmer et al. \cite{gilmer2017} refers to this task as \textit{readout}.

Performing vertex pooling is made challenging by the fact that different sets of vertex representations corresponding to different graphs may contain different numbers of elements. Furthermore, the elements in a given set are unordered. Therefore one cannot directly apply a feed-forward or recurrent architecture because these require an input lying in a vector space or sequence space respectively. To overcome this challenge most solutions involve mapping the sets of vertex representations to either a vector or sequence space which can then form the input to a feed-forward or recurrent architecture respectively. There exists a wide array of such solutions ranging from computing simple summary statistics such as mean vertex representation to more complex clustering based methods \cite{ying2018hierarchical}. 

In this article we propose a novel binary graph classification model which performs vertex pooling by mapping a set of vertex representations to an element in a reproducing kernel Hilbert space (RKHS). A RKHS is a function space for which there exists a corresponding kernel function equalling the dot product in this space. Being a function space where the domain of functions in this space is a Euclidean Space, the RKHS in question is of infinite dimension and in turn has high model capacity. However, the infinite nature of this space makes it challenging to work directly in this space. To overcome this challenge, we use the corresponding kernel function which allows us to implicitly compute the dot product in this space without explicitly mapping to the space in question. This is a commonly used strategy known as the \textit{kernel trick}. More specifically, the kernel corresponding to the RKHS is used within a support vector machine (SVM) to perform binary graph classification. A useful feature of the proposed pooling method is that the mapping to a RKHS is parameterized by a scale parameter which controls the degree to which different sets of vertex representations can be discriminated.

The proposed graph classification model is trained in a supervised end-to-end manner where the convolutional layers, the kernel function and SVM parameters are jointly optimized with respect to a regularized classification loss. This approach is distinct from existing kernel-based models which instead use feature engineering or unsupervised learning to define the kernel function and only optimize the parameters of the classification method in a supervised manner \cite{yanardag2015deep}. Using feature engineering can result in \textit{diagonal dominance} whereby a graph is determined to only be similar to itself, but not to any other graph \cite{yanardag2015deep}. Although unsupervised learning can overcome this problem and improve performance, the kernel may not be optimal for the task at hand given it was learned in an unsupervised as opposed to supervised manner \cite{pmlr-v80-ivanov18a}. The proposed solution of optimizing in an end-to-end manner overcomes these limitations.

The remainder of this paper is structured as follows. Section \ref{sec:related_works} reviews related work on graph kernels and vertex pooling methods. Section \ref{sec:methodology} describes the proposed graph classification model. Section \ref{sec:results} presents an evaluation of this model through comparison to 12 baseline models on 4 datasets. Finally, section \ref{sec:conclusions} draws some conclusions from this work and discusses possible future research directions.

\section{Related Work}
\label{sec:related_works}
In this work we propose a novel vertex pooling method which performs vertex pooling by mapping to a RKHS. In the following two sections we review related work on vertex pooling methods and graph kernels.

\subsection{Vertex Pooling}
As discussed in the introduction to this article, existing vertex pooling methods generally map the set of vertex representations to a fixed dimensional vector space or sequence space. The simplest methods for performing vertex pooling compute a summary statistic of the set of vertex representations. Commonly used summary statistics include mean, max and sum \cite{duvenaud2015}. Despite the simple nature of these methods, a recent study by Luzhnica et al. \cite{luzhnica2019graph} demonstrated that in some cases they can outperform more complex methods. Zhang et al. \cite{zhang2018} proposed a vertex pooling method which first performs a sorting of vertex representations based on the Weisfeiler-Lehman graph isomorphism algorithm. A subset of these vertex representations are then selected based on this ranking, where the size of this subset is a user specified parameter. Tarlow et al. \cite{li2015gated} proposed a vertex pooling method which outputs an element in sequence space. Gilmer et al. \cite{gilmer2017} proposed to perform vertex pooling by applying the set2set model from Vinyals et al. \cite{vinyals2015}. The set2set model maps the set of vertex representations to fixed dimensional vector space representation which is invariant to the order of elements in the set. Ying et al. \cite{ying2018hierarchical} proposed a vertex pooling method which uses clustering to iteratively integrate vertex representations and outputs an element in a fixed dimensional vector space. Kearnes et al. \cite{kearnes2016} proposed a vertex pooling method which creates a fuzzy histogram of the vertex representations and outputs an element in a fixed dimensional vector space.

\subsection{Graph Kernels}
As described in the introduction to this article, existing kernel-based graph classification methods use either feature engineering or unsupervised learning to define the kernel. We now review each of these approaches in turn.

The most common approach for feature engineering kernels is the $\mathcal{R}$-convolution framework where the kernel function of two graphs is defined in terms of the similarity of their respective substructures \cite{haussler1999convolution}. This framework is similar to the bag-of-words framework used in natural language processing. Substructures used in the $\mathcal{R}$-convolution framework to define kernels include graphlets \cite{shervashidze2009efficient}, shortest path properties \cite{borgwardt2005shortest} and random walk properties \cite{sugiyama2015halting}.

The Weisfeiler-Lehman framework is a framework for feature engineering kernels which is inspired by the Weisfeiler-Lehman test of graph isomorphism. In this framework the vertex representations of a given graph are iteratively updated in a similar manner to graph convolution to give a sequence of graphs. A kernel is then defined with respect to this sequence by summing the application of a given kernel, known as the base kernel, to each graph in the sequence. Shervashidze et al. \cite{shervashidze2011weisfeiler} proposed a family of kernels using this framework by considering a set of base kernels including one which measures the similarity of shortest path properties. Rieck et al. \cite{rieck2019persistent} proposed a kernel using this framework by considering a base kernel which measures the similarity of topological properties.

Kriege et al. \cite{kriege2016valid} proposed another framework for feature engineering kernels known as \textit{assignment kernels} which computes an optimal assignment between graph substructures and sums over a kernel applied to each correspondence in the assignment. The authors proposed a number of kernels using this framework including one based on the Weisfeiler-Lehman graph isomorphism algorithm. Kondor et al. \cite{kondor2016multiscale} proposed a multiscale kernel which considers vertex features plus topological information through the graph Laplacian. Zhang et al. \cite{zhang2018retgk} proposed a kernel-based on the return probabilities of random walks. The authors used an approximation of the kernel function so that the method can be applied to large datasets \cite{rahimi2008random}.

To overcome the limitations of feature engineering and improve performance, recent works in the field of graph kernels have considered unsupervised learning techniques. These methods generally learn a graph representation in an unsupervised manner and subsequently use this representation to define a kernel. Yanardag et al. \cite{yanardag2015deep} proposed a kernel which uses the $\mathcal{R}$-convolution framework to define a set of substructures and subsequently learns an embedding of these substructures in an unsupervised manner using a word2vec type model.
Ivanov et al. \cite{pmlr-v80-ivanov18a} proposed a kernel which determines two graphs to be similar if their vertices have similar neighbourhoods measured in terms of anonymous walks which are a generalization of random walks. Learning is performed in an unsupervised manner using a word2vec type model. Nikolentzos et al. \cite{nikolentzos2017matching} proposed a graph kernel which first computes sets of vertex representations corresponding to the graphs in question in an unsupervised manner. The similarity of these sets are then computed using the earth mover's distance. The authors noted that these similarities do not yield a positive semidefinite kernel matrix preventing it from being used in some kernel-based classification methods. To overcome this issue the authors use a version of the support vector machine for indefinite kernel matrices. Similar to Nikolentzos et al. \cite{nikolentzos2017matching}, Wu et al. \cite{wu2019scalable} proposed a graph kernel which first computes sets of vertex representations corresponding to the graphs in questions in an unsupervised manner. The resulting set of embeddings are in turn used to embed the graph in question by measuring the disturbance distance to sets of embeddings corresponding to random graphs. Finally, this graph representation is used to define a kernel. Nikolentzos et al. \cite{nikolentzos2018kernel} proposed a method that performs an unsupervised clustering of the input graph into components and subsequently learns a kernel function which takes as input these components.

\section{Methodology}
\label{sec:methodology}
The proposed graph classification model consists of the following three steps. In the first step, a sequence of graph convolutional layers are applied to the graph in question to generate a corresponding set of vertex representations. In the second step, this set of vertex representations is mapped to a RKHS. In the final step, graph classification is performed using a SVM. Each of these three steps are described in turn in the first three subsections of this section. In the final subsection we describe how the parameters of each step are optimized jointly in an end-to-end manner. Before that, we first introduce some notation and formally define the problem of graph classification.

A graph is a tuple $(V,E)$ where $V$ is a set of vertices and $E \subseteq (V \times V)$ is a set of edges. Let $\mathcal{G}$ denote the space of graphs. Let $l: V \rightarrow \Sigma$ denote a vertex labelling function. In this work we assume that $\Sigma$ is a finite set. Let $\mathbb{G} = \lbrace \mathcal{G}_1, \mathcal{G}_2, \dots, \mathcal{G}_n \rbrace$ denote a set of $n$ graphs and $\mathbb{Y} = \lbrace \mathcal{Y}_1, \mathcal{Y}_2, \dots, \mathcal{Y}_n \rbrace$ denote a corresponding set of graph labels. In this work we assume that graph labels take elements in the set $\lbrace 0, 1 \rbrace$. In this work we consider the problem of binary graph classification where given $\mathbb{G}$ and $\mathbb{Y}$ we wish to learn a map $\mathcal{G} \rightarrow \lbrace 0, 1\rbrace$.

\subsection{Graph Convolution Layers}
A large number of different graph convolutional layers have been proposed. Broadly speaking a graph convolutional layer will update the representation of each vertex in a given graph where this update is a function of the current representation of that vertex plus the representations of its adjacent neighbours. In this section we only briefly review existing graph convolutional layers but the interested reader can find a more indepth analysis in the following review papers \cite{zhang2018deep,wu2019comprehensive}.

Gilmer et al. \cite{gilmer2017} showed that many different convolutional layers may be reformulated in terms of a framework called \textit{Message Passing Neural Networks} defined in terms of a message function $M$ and an update function $U$. In this framework vertex representations are updated according to Equation \ref{eq:message_passing_1} where $h_v^t$ denotes the representation of vertex $v$ output from the $t$-th convolutional layer and $N(v)$ denotes the set of vertices adjacent to $v$. Each vertex representation $h_v^t$ is an element of $\mathbb{R}^m$ where the dimension $m$ may vary from layer to layer. For the input layer, that is $t=1$, vertex representations equal a one-hot encoding of the vertex labelling function $l$ and therefore the corresponding dimension is $\vert \Sigma \vert$. For all subsequent layers the corresponding dimension is a model hyper-parameter.

\begin{equation}
\label{eq:message_passing_1}
\begin{split}
m^{t+1}_v & = \sum_{w \in N(v)} M(h_v^t, h_w^t) \\
h^{t+1}_v & = U(h^{t}_v, m^{t+1}_v)
\end{split}
\end{equation}

In the proposed graph classification model we use the functions $M$ and $U$ originally proposed by Hamilton et al. \cite{hamilton2017} and defined in Equation \ref{eq:message_passing_2}. Here CONCAT is the horizontal vector concatenation operation, $\mathbf{W}^t$ and $\mathbf{b}^t$ are the weights and biases respectively for the $t$-th convolutional layer, and ReLU is the real valued rectified linear unit non-linearity.

\begin{equation}
\label{eq:message_passing_2}
\begin{split}
M(h_v^t, h_w^t) & =  h_w^t \\
U(h^{t}_v, m^{t+1}_v) & = \text{ReLU} \left( \mathbf{W}^t \cdot \text{CONCAT} \left( h_v^{t-1}, m^{t+1}_v \right) + \mathbf{b}^t \right)
\end{split}
\end{equation}

A sequence of two convolutional layers were used in the proposed model. A number of studies have found that the use of two layers empirically gives the best performance \cite{kipf2017semi}. This sequence of layers will map a graph $\mathcal{G}_i=(V,E)$ to a set of $\vert V \vert$ points in $\mathbb{R}^m$ where $m$ is the dimension of the final convolutional layer. Since the number of vertices in a graph may vary the number of points in $\mathbb{R}^m$ may in turn vary. Let us denote by Set the space of sets of points in $\mathbb{R}^m$. Given this, the sequence of convolutions layers defines a map $\mathcal{G} \rightarrow \text{Set}$.

\subsection{Mapping to RKHS}
The output from the sequence of convolutional layers defined in the previous subsection is an element in the space Set. In this section we propose a method for mapping elements in this space to a reproducing kernel Hilbert space (RKHS). We in turn define a kernel between elements in this space.

A Hilbert space is a vector space with an inner product such that the induced norm turns the space into a complete metric space. A positive-semidefinite kernel on a set $\mathcal{X}$ is a function $k: \mathcal{X} \times \mathcal{X} \rightarrow \mathbb{R}$ such that there exists a feature space $\mathcal{H}$ and a map $\phi: \mathcal{X} \rightarrow \mathcal{H}$ such that $k \left(x,y \right) = \langle \phi(x) , \phi(y) \rangle$ where $x,y \in \mathcal{X}$ and $\langle \cdot , \cdot \rangle$ denotes the dot product in $\mathcal{H}$. Equivalently, a function $k: \mathcal{X} \times \mathcal{X} \rightarrow \mathbb{R}$ is a kernel if and only if for every subset $\left\lbrace x_1, \dots, x_q \right\rbrace \subseteq \mathcal{X}$, the $q \times q$ matrix $K$ with entries $K_{ij}=k(x_i, x_j)$ is positive semi-definite \cite{scholkopf2002learning}. Given a kernel $k$, one can define a map $\mathcal{X} \to \mathbb{R}^\mathcal{X}$ as Equation \ref{eq:kernel_function_space} where codomain of this map is the space of real valued functions on $\mathcal{X}$. Such a space is called a function space. Given this, it can be proven that $k(x,y) = \langle k(\cdot, x) , k(\cdot, y) \rangle$. By virtue of this property, $\mathbb{R}^\mathcal{X}$ is called a \textit{reproducing kernel Hilbert space} (RKHS) corresponding to the kernel $k$ \cite{scholkopf2002learning}.

\begin{equation}
\label{eq:kernel_function_space}	
x \mapsto k(\cdot, x)
\end{equation}

Let $k^R_\sigma: \mathbb{R}^m \times \mathbb{R}^m \rightarrow \mathbb{R}$ be the Gaussian kernel function defined in Equation \ref{eq:gaussian} which is parameterized by $\sigma \in \mathbb{R}_{\geq0}$. 

\begin{equation}
\label{eq:gaussian}
k^R_\sigma(u,v) = \exp(-\Vert u-v \Vert^2_2 /2\sigma^2)
\end{equation}

Given $k^R_\sigma$, we define a map $F: \text{Set} \times \mathbb{R} \rightarrow \mathbb{R}^{\mathbb{R}^m}$ in Equation \ref{eq:function_space} where $\mathbb{R}^{\mathbb{R}^m}$ is the space of real valued functions on $\mathbb{R}^m$. To illustrate this map consider the element of Set displayed in Figure \ref{fig:function_space_1} where the dimension $m$ equals 2. Recall that elements in the space Set correspond to sets of points in $\mathbb{R}^m$. Figures \ref{fig:function_space_2} and \ref{fig:function_space_3} display the elements of $\mathbb{R}^{\mathbb{R}^m}$ resulting from applying the map $F$ to this element of Set with $\sigma$ parameter values of $0.001$ and $0.0005$ respectively.

\begin{equation}
\label{eq:function_space}
F(x, \sigma) = \sum_{v \in x} k^R_\sigma(v,\cdot)
\end{equation}

The parameter $\sigma$ of the map $F$ is a scale parameter and may be interpreted as follows. As the value of $\sigma$ approaches $0$, $F(x, \sigma)$ becomes a sum of a set indicator functions applied to $x$. In this case distinct elements of the space Set map to distinct elements of $\mathbb{R}^{\mathbb{R}^m}$ where the distance between these functions measured by the $L^p$ norm is greater than zero. On the other hand, as $\sigma$ approaches $\infty$, differences between the functions are gradually smoothed out and in turn the distance between the functions gradually reduces. Therefore, one can view the parameter $\sigma$ as controlling the discrimination power of the method. 

\begin{figure*}
\begin{center}
\subfigure[]{\includegraphics[height=3.5cm]{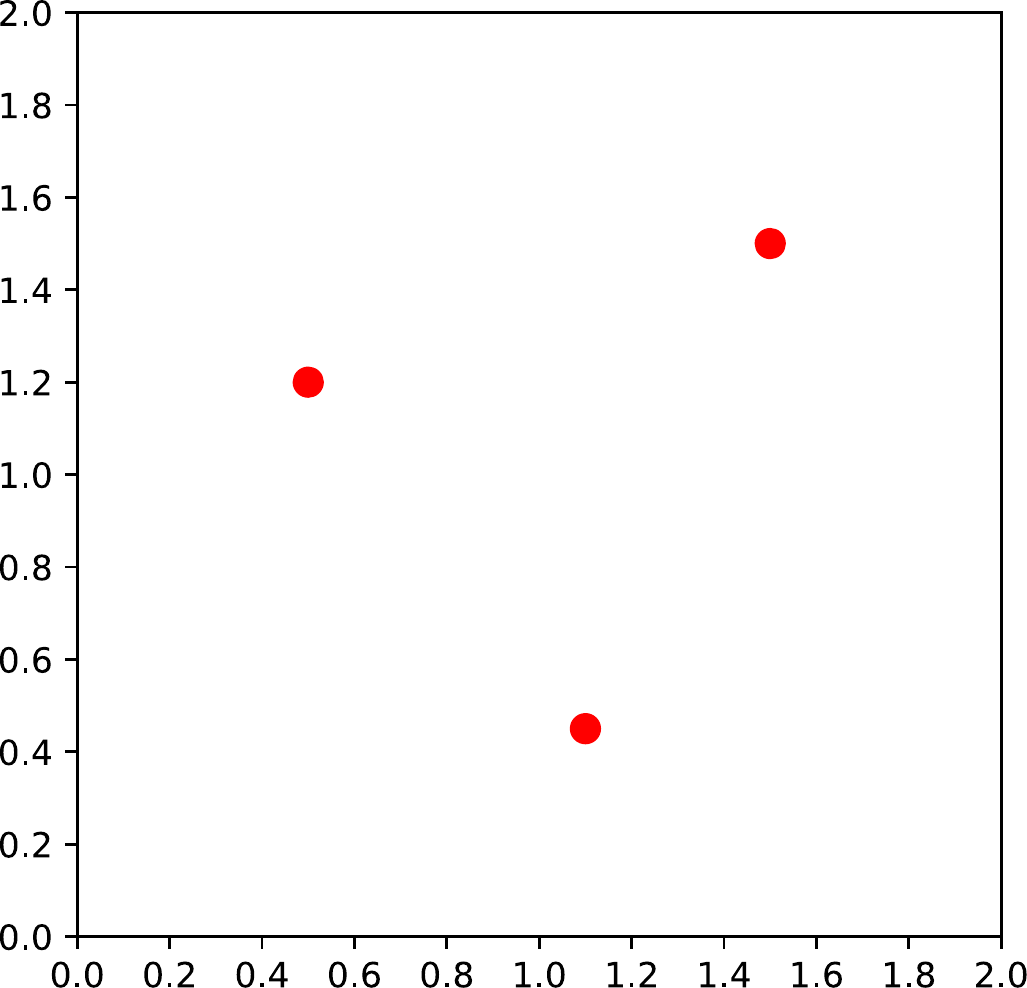}
\label{fig:function_space_1}}
\hspace{.2cm}
\subfigure[]{\includegraphics[height=3.5cm]{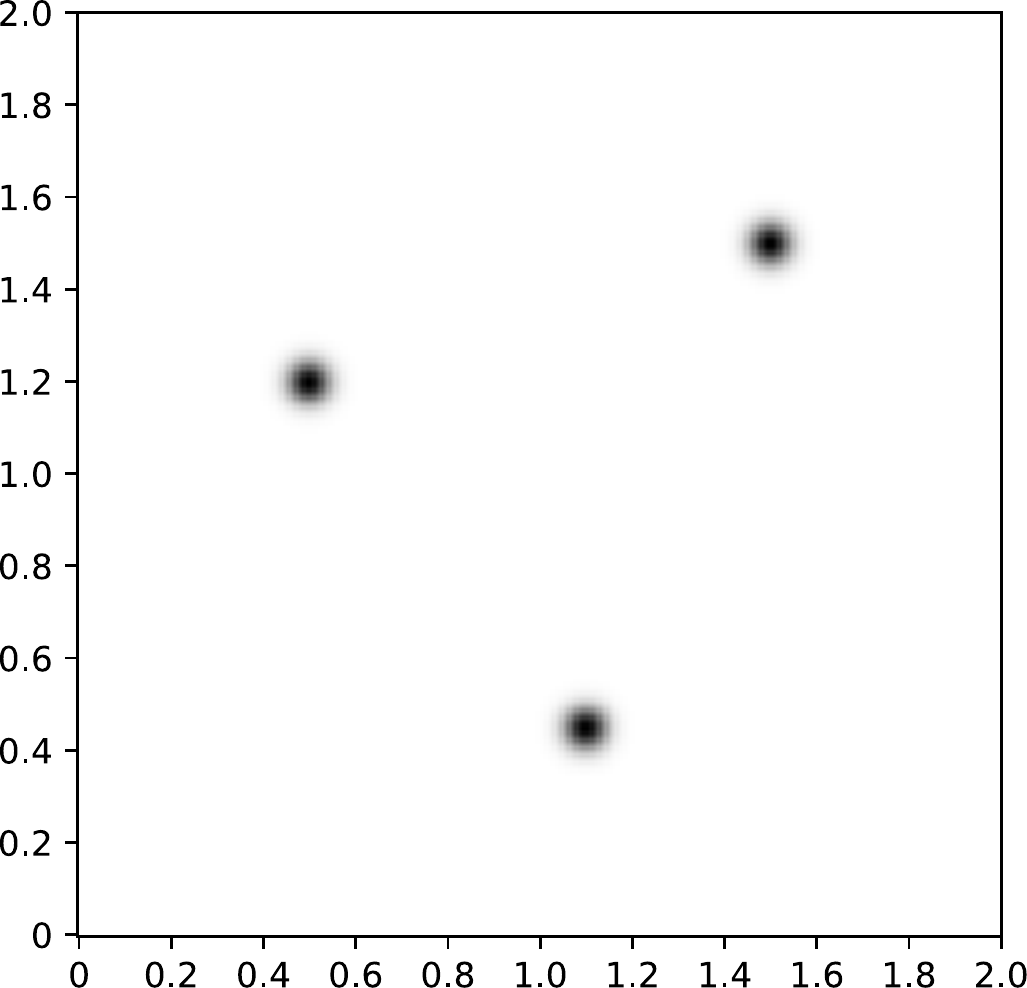}
\label{fig:function_space_2}}
\hspace{.2cm}
\subfigure[]{\includegraphics[height=3.5cm]{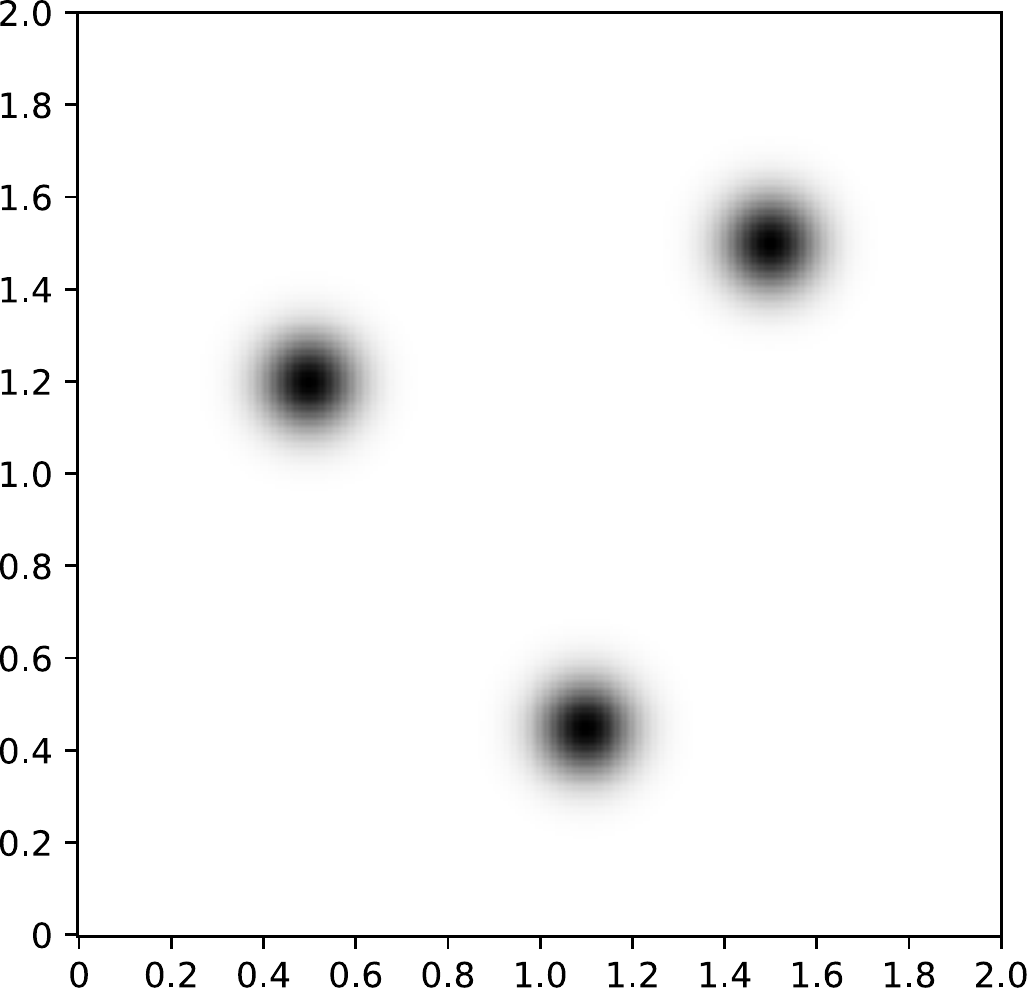}
\label{fig:function_space_3}}
\caption{An element of the space Set is displayed in (a) where the dimension $m$ equals 2 and each point in $\mathbb{R}^m$ is represented by a red dot. The elements of $\mathbb{R}^{\mathbb{R}^m}$, which are themselves functions $\mathbb{R}^{m} \rightarrow \mathbb{R}$, resulting from applying the map $F$ to this element with $\sigma$ parameter values of $0.001$ and $0.0005$ are displayed in (b) and (c) respectively.}
\label{fig:intro_shapes_single_component}
\end{center}
\end{figure*}

Given the map $F$ defined in Equation \ref{eq:function_space}, we define the kernel $k^L_\sigma: \mathbb{R}^{\mathbb{R}^m} \times \mathbb{R}^{\mathbb{R}^m} \rightarrow \mathbb{R}$ in Equation \ref{eq:function_kernel}. Note that, the final equality in this equation follows from the reproducing property of the RKHS related to $k^R_\sigma$ and the bilinearity of the inner product \cite{paulsen2016introduction}. By examination of Equation \ref{eq:function_kernel}, we see that the kernel $k^L_\sigma$ equals the dot product between elements in the codomain of the map $F$ which is an infinite dimensional function space. That is, the kernel allows us to operate in this codomain without the computational complexity of explicitly mapping into it. In Theorem \ref{them:1} we prove that $k^L_\sigma$ is a valid positive-semidefinite kernel.

\begin{equation}
\label{eq:function_kernel}
\begin{split}
& k^L_\sigma(F(x_i, \sigma), F(x_j, \sigma)) = \langle F(x_i, \sigma), F(x_j, \sigma) \rangle \\
& = \langle \sum_{v \in x_i} k^R_\sigma(v,\cdot), \sum_{u \in x_j} k^R_\sigma(u,\cdot) \rangle = \sum_{v \in x_i} \sum_{u \in x_j} k^R_\sigma(v, u)
\end{split}
\end{equation}

\begin{theorem}
\label{them:1}
The kernel $k^L_\sigma$ is a positive-semidefinite kernel.
\end{theorem}
 
\begin{proof}
The kernel $k^L_\sigma$ is a positive-semidefinite kernel because it is defined in Equation \ref{eq:function_kernel} to equal the dot product in the space $\mathbb{R}^{\mathbb{R}^m}$.
\end{proof}

The kernel $k^L_\sigma$ has a specific scale which is specified by $\sigma$. In order to adopt a multi-scale approach we consider a set of $s$ scales $\Sigma = \lbrace \sigma_1, \dots, \sigma_s \rbrace$ to define a corresponding set of kernels $\lbrace k^L_{\sigma_1}, \dots, k^L_{\sigma_s} \rbrace$. We combine these kernels using a linear combination defined in Equation \ref{eq:function_kernel_sum} where $\lbrace \beta_1, \dots, \beta_s \rbrace \in \mathbb{R}^s_{\geq0}$. In Theorem \ref{them:2} we prove that $k^L_\Sigma$ is a valid positive-semidefinite kernel.

\begin{equation}
\label{eq:function_kernel_sum}
k^L_\Sigma(F(x_i), F(x_j)) = \sum_{l=1}^{s} \beta_l k^L_{\sigma_l}(F(x_i), F(x_j))
\end{equation}

\begin{theorem}
\label{them:2}
The kernel $k^L_\Sigma$ is a positive-semidefinite kernel.
\end{theorem}
 
\begin{proof}
The kernel $k^L_\Sigma$ is a positive-semidefinite kernel because it is the sum of positive-semidefinite kernels and the coefficients $\lbrace \beta_1, \dots, \beta_s \rbrace$ are all positive (see proposition 13.1 in \cite{scholkopf2002learning}).
\end{proof}

\subsection{SVM}
Recall that we consider the problem of graph classification whereby given $\mathbb{G} = \lbrace \mathcal{G}_1, \mathcal{G}_2, \dots, \mathcal{G}_n \rbrace$ and $\mathbb{Y} = \lbrace \mathcal{Y}_1, \mathcal{Y}_2, \dots, \mathcal{Y}_n \rbrace$ we wish to learn a map $\mathcal{G} \rightarrow \lbrace 0, 1\rbrace$.


Let $f: \text{Set} \rightarrow \mathbb{R}$ be a map from which we obtain a decision function by sgn$(f)$. That is, if $f$ returns a positive value we classify the graph in question as $1$ and otherwise we classify it as $0$. We determine a suitable map $f$ lying in the RKHS $\mathcal{H}$ corresponding to the kernel $k^L_\Sigma$ by Equation \ref{eq:opt_f_1}. Note that, the first term in this sum corresponds to the \textit{soft margin loss} \cite{scholkopf2002learning} and the second term is a regularization term.

\begin{equation}
\label{eq:opt_f_1}
\hat{f} = \argmin_{f \in \mathcal{H}} \sum_{i=1}^{n} \max(0, 1-y_if(x_i)) + \lambda \Vert f \Vert_\mathcal{H}^2
\end{equation}

By the \textit{representer theorem} any solution to Equation \ref{eq:opt_f_1} can be written in the form of Equation \ref{eq:f_representer} where $\lbrace \alpha_1, \dots, \alpha_n \rbrace \in \mathbb{R}^n$ \cite{paulsen2016introduction}.

\begin{equation}
\label{eq:f_representer}
f(\cdot) = \sum_{j=1}^{n} \alpha_j k^L_\Sigma(\cdot, x_j)
\end{equation}

Substituting this into Equation \ref{eq:opt_f_1} we obtain Equation \ref{eq:opt_f_2} where optimization of the function $f$ is performed with respect to $\lbrace \alpha_1, \dots, \alpha_n \rbrace \in \mathbb{R}^n$. Here $K^L_{i,j} = k^L_\Sigma(x_i, x_j)$, $\odot$ is the elementwise multiplication operator (Hadamard product), $\vec{0}$ is a vector of zeros of size $n$ and $\vec{1}$ is a vector of ones of size $n$.

\begin{equation} \label{eq:opt_f_2}
\begin{split}
\hat{f} & = \argmin_{\alpha \in \mathbb{R}^n} \sum_{i=1}^{n} \max(0, 1-y_i \sum_{j=1}^{n} \alpha_j k^L_\Sigma(x_i, x_j)) \\
 & \qquad \qquad + \lambda \Vert \sum_{j=1}^{n} \alpha_j k^L_\Sigma(\cdot, x_j) \Vert_\mathcal{H}^2 \\
 & = \argmin_{\alpha \in \mathbb{R}^n} \sum_{i=1}^{n} \max(0, 1-y_i \sum_{j=1}^{n} \alpha_j k^L_\Sigma(x_i, x_j)) \\
 & \qquad \qquad + \lambda \sum_{i,j=1}^{n} \alpha_i \alpha_j k^L_\Sigma(x_i, x_j) \\
 & = \argmin_{\alpha \in \mathbb{R}^n} \Vert \max(\vec{0}, \vec{1} - y \odot K^L \alpha) \Vert_1 \\
 & \qquad \qquad + \lambda \alpha^T K^L \alpha
\end{split}
\end{equation}

\subsection{End-to-End Optimization}
As described in the previous subsections, the proposed classification model contains three steps with each having corresponding parameters which require optimization with respect to the objective function defined in Equation \ref{eq:opt_f_2}. The parameters in question are the sets of convolutional layer parameters $\mathbf{W}^t$ and $\mathbf{b}^t$ defined in Equation \ref{eq:message_passing_2}, the sets of kernel parameters $\sigma_l$ and $\beta_l$ defined in Equation \ref{eq:function_kernel_sum}, and the set of SVM parameters $\alpha_j$ defined in Equation \ref{eq:f_representer}. All of these parameters are unconstrained real values apart from the sets of kernel parameters $\sigma_l$ and $\beta_l$ which are constrained to be positive real values. As such, the optimization problem in question is a constrained optimization problem. In this work we wish to optimize all the above model parameters jointly in an end-to-end manner. We refer to this as the end-to-end optimization problem. Note that, if only the SVM parameters were optimized and all other parameters were fixed, the optimization problem could be formulated as a quadratic program by taking the dual and solved in closed-form \cite{scholkopf2002learning}. This is the most commonly used method for optimizing the parameters of an SVM. 

In order to solve the end-to-end optimization problem we use a gradient based optimization method. Such methods are the most commonly used methods for optimizing neural network parameters \cite{goodfellow2016deep}. There are two main approaches that can be used to apply a gradient based optimization method to a constrained optimization problem. The first approach is to project the result of each gradient step back into the feasible region. The second approach is to transform the constrained optimization problem into an unconstrained optimization problem and solve this problem. Such a transformation can be achieved using the Karush-Kuhn-Tucker (KKT) method \cite{nocedal2006numerical}. In this work we use the former approach. In practice this reduces to passing the parameters $\sigma_l$ and $\beta_l$ through the function $\max(\cdot, 0)$ after each gradient step. The above optimization can be used in conjunction with any gradient based optimization method such as stochastic gradient descent. In this work the Adam method was used \cite{kingma2014adam}. 

\section{Evaluation}
\label{sec:results}
In this section we present an evaluation of the proposed end-to-end graph classification model with respect to current state-of-the-art models. This section is structured as follows. Section \ref{sec:results_imple} provides implementation details for the proposed model. Section \ref{sec:results_bench} describes the baseline models used to compare the proposed model against. Finally, section \ref{sec:results_datasets} describes the datasets used in this evaluation and compares the performance of all models on these datasets.

\subsection{Implementation Details}
\label{sec:results_imple}
The parameters of the proposed model were initialized as follows. The convolutional layer weights $\mathbf{W}^t$ and biases $\mathbf{b}^t$ in Equation \ref{eq:message_passing_2} were initialized using Kaiming initialization \cite{he2015delving} and to a value of $0$ respectively. The kernel parameters $\lbrace \sigma_1, \dots, \sigma_s \rbrace$ and $\lbrace \beta_1, \dots, \beta_s \rbrace$ and $s$ in Equation \ref{eq:function_kernel_sum} were all initialized to a value of $1$.

The model hyper-parameters were set as follows. The dimension of the convolutional hidden layers was set equal to $25$. The Adam optimizer learning rate was set to its default value of $0.001$ and training was performed for 300 epochs. The hyper-parameters $\lambda$ in Equation \ref{eq:opt_f_2} and $s$ in Equation \ref{eq:function_kernel_sum} were selected from the sets $\lbrace 0.0, 0.5, 1.0, 1.5, 2.0, 2.5, 3.0 \rbrace$ and $\lbrace 1,2 \rbrace$ respectively by considering classification accuracy on a validation set. Larger values for the hyper-parameter $s$ were not considered to ensure scalability of the model to medium sized datasets.


The time and space complexity of classifying a given graph is O($n$) where $n$ is the number of graphs in the training dataset. This is a consequence of the summation in Equation \ref{eq:f_representer} over all training examples.
The time and space complexity of performing an update of the method parameters using backprop is O($n^2$) because this step computes the complete kernel matrix $K$ in Equation \ref{eq:opt_f_2}. Note that, the above time complexity analysis assumes that each element in the kernel matrix $K$ can be computed in constant time. In reality, if we assume that each graph contains $m$ vertices the time complexity of computing each element in $K$ is $m^2$ (see Equation \ref{eq:function_kernel}). In this case the time complexity of classifying a given graph is O($n m^2$) while the time complexity of performing an update of the method parameters using backprop is O($n^2 m^2$).

\subsection{Baseline Methods}
\label{sec:results_bench}
As described in the related work section of this paper, existing models for graph classification belong to two main categories of feature engineered kernel and end-to-end deep learning models. For the purposes of this evaluation, we compared the model proposed in this work to baseline models in each of these categories.

A set of 17 baseline models were considered where this set contains 5 feature engineered kernel models and 12 end-to-end deep learning models. We considered so many baseline models to ensure we were comparing to state of the art; many existing models claim to outperform each other so it is difficult to determine which models are in fact state of the art. The end-to-end deep learning baseline models considered in the evaluation are end-to-end models but not are kernel-based models. The proposed model is the first end-to-end kernel-based model for graphs.

In order to cover the breadth of different feature engineered kernel models, we considered three kernel functions in the $\mathcal{R}$-convolution framework, one kernel function in the Weisfeiler-Lehman framework and one kernel function which uses unsupervised learning. The kernel functions in question are entitled Graphlet by Shervashidze et al. \cite{shervashidze2009efficient}, Shortest Path by Borgwardt et al. \cite{borgwardt2005shortest}, Vertex Histogram by Sugiyama et al. \cite{sugiyama2015halting}, Weisfeiler Lehman by Shervashidze et al. \cite{shervashidze2011weisfeiler} and Pyramid Match by Nikolentzos et al. \cite{nikolentzos2017matching}. These kernel functions are described in the Appendix section of this article. For each kernel function, classification was performed using a Support Vector Machine (SVM) with a kernel matrix precomputed using the kernel function in question. Implementations for the kernel functions were obtained from the \textit{GraKeL} Python library \cite{JMLR:v21:18-370}.

The end-to-end deep learning baseline models considered are entitled GCN, GCNWithJK, GIN, GIN0, GINWithJK, GIN0WithJK, GraphSAGE, GraphSAGEWithJK, DiffPool, GlobalAttentionNet, Set2SetNet and SortPool. The architectures of these models are described in the Appendix section of this article. Implementations for these models were obtained from the \textit{PyTorch Geometric} Python library \cite{fey2019}; these can be downloaded directly from the benchmark section of the PyTorch Geometric website \footnote{\url{https://github.com/rusty1s/pytorch_geometric/tree/master/benchmark}}.  For each end-to-end deep learning baseline model the corresponding model parameters were optimized using the Adam optimizer with the default learning rate of $0.001$ and run for 300 epochs. In all cases a negative log likelihood loss function was used. Model hyper-parameters corresponding to the number and dimension of hidden layers were selected from the sets $\lbrace 1, 2, 3, 4, 5 \rbrace$ and $\lbrace 16, 32, 64, 128 \rbrace$ respectively by considering the loss on a validation set.

\subsection{Datasets and Results}
\label{sec:results_datasets}
To evaluate the proposed graph classification model we considered five commonly used graph classification datasets obtained from the TU Dortmund University graph dataset repository \cite{KKMMN2016} \footnote{\url{https://chrsmrrs.github.io/datasets/}}. Summary statistics for each of these datasets are displayed in Table \ref{table:datasets}. The MUTAG dataset contains graphs corresponding to chemical compounds and the binary classification problem concerns predicting a particular characteristic of the chemical \cite{debnath1991}. The PTC\_MR dataset contains graphs corresponding to chemical compounds and the binary classification problem concerns predicting a carcinogenicity property. The BZR\_MD dataset contains graphs corresponding to chemical compounds and the binary classification problem concerns predicting a particular characteristic of the chemical \cite{sutherland2003spline}. The PTC\_FM dataset contains graphs corresponding to chemical compounds and the binary classification problem and concerns predicting a carcinogenicity property. The COX2 dataset contains graphs corresponding to molecules and the binary classification problem concerns predicting if a given molecule is active or inactive \cite{sutherland2003spline}.

\begin{table*}[h!]
\centering
\begin{tabular}{c | c | c | c | c | c} 
 & \textbf{MUTAG} & \textbf{PTC\_MR} & \textbf{BZR\_MD} & \textbf{PTC\_FM} & \textbf{COX2} \\ [0.5ex] 
 \hline
 Number graphs & 188 & 344 & 306 & 349 & 467 \\
 Mean number vertices & 17.93 & 14.29 & 35.75 & 14.11 & 41.22 \\
 Mean number edges & 19.79 & 14.69 & 38.36 & 14.48 & 43.45 \\
 \hline
\end{tabular}
\caption{Summary statistics for each dataset used in the evaluation.}
\label{table:datasets}
\end{table*}

Stratified $k$-folds cross-validation with a $k$ value of $10$ was used to split the data into training, validation and testing sets. During each of the $k$ training steps, one of the $k-1$ folds in the training set was randomly selected to be a validation set and classification accuracy on this set was used to select model hyper-parameters. It has been shown that different training, validation and testing set splits of the data can lead to quite different rankings of graph classification models \cite{shchur2018pitfalls}. However averaging the performance of $k$ different splits, as done in this work, helps to reduce this instability. In our analysis the same training, testing and validation splits were used for all graph classification models considered. This is an important point because the performance of a given model may vary as a function of the split used. For each dataset we computed the mean accuracy on the test sets for each method. The results of this analysis are displayed in Table \ref{table:mean_accuracy}. For two of the five datasets, the proposed graph classification model achieved the best mean performance and outperformed most other models by a significant margin. For the remaining three datasets, the proposed method achieved a better mean performance than many but not all baseline methods. These positive results demonstrate the utility of the proposed model. In most cases the proposed model achieved best performance on the validation set with the hyper-parameter $s$ having a value of $2$. Recall, from Equation \ref{eq:function_kernel_sum}, that this hyper-parameter equals the number of individual kernels integrated by the model. This demonstrates the utility of integrating multiple kernels.

\begin{table*}[h!]
\centering
\begin{tabular}{c | c | c | c | c | c} 
 \textbf{Model} & \textbf{MUTAG} & \textbf{PTC\_MR} & \textbf{BZR\_MD} & \textbf{PTC\_FM} & \textbf{COX2} \\ [0.5ex] 
 \hline
 Graphlet &  0.86 $\pm$ 0.05  & 0.54 $\pm$ 0.08 & 0.61 $\pm$ 0.12 & 0.57 $\pm$ 0.07 & 0.75 $\pm$ 0.06 \\
 Shortest Path & 0.82 $\pm$ 0.06 & 0.55 $\pm$ 0.08 & 0.68 $\pm$ 0.07 & 0.55 $\pm$ 0.07 & 0.77 $\pm$ 0.06 \\
 Vertex Histogram & 0.85 $\pm$ 0.05 & 0.58 $\pm$ 0.07 & \textbf{0.70 $\pm$ 0.08} & 0.58 $\pm$ 0.06 & 0.75 $\pm$ 0.06 \\
 Weisfeiler Lehman & 0.71 $\pm$ 0.06 & 0.59 $\pm$ 0.07 & 0.60 $\pm$ 0.07 & 0.63 $\pm$ 0.08 & 0.65 $\pm$ 0.08 \\
 Pyramid Match & 0.86 $\pm$ 0.06 & 0.54 $\pm$ 0.06 & 0.62 $\pm$ 0.07 & 0.57 $\pm$ 0.09 & 0.73 $\pm$ 0.06 \\ \\
 
 GCN & 0.73 $\pm$ 0.06 & 0.57 $\pm$ 0.03 & 0.68 $\pm$ 0.06 & 0.61 $\pm$ 0.08 & 0.79 $\pm$ 0.03 \\
 GCNWithJK & 0.73 $\pm$ 0.07 & 0.57 $\pm$ 0.05 & 0.68 $\pm$ 0.07 & 0.62 $\pm$ 0.08 & 0.78 $\pm$ 0.02 \\
 GIN & 0.82 $\pm$ 0.07 & 0.54 $\pm$ 0.05 & 0.62 $\pm$ 0.09 & 0.57 $\pm$ 0.07 & 0.80 $\pm$ 0.04 \\
 GIN0 & 0.85 $\pm$ 0.04 & 0.57 $\pm$ 0.08 & 0.63 $\pm$ 0.13 & 0.59 $\pm$ 0.05 & 0.77 $\pm$ 0.04 \\
 GINWithJK & 0.83 $\pm$ 0.07 & 0.55 $\pm$ 0.07 & 0.61 $\pm$ 0.15 & 0.60 $\pm$ 0.04 & 0.80 $\pm$ 0.06 \\
 GIN0WithJK & 0.83 $\pm$ 0.06 & 0.54 $\pm$ 0.07 & 0.63 $\pm$ 0.10 & 0.59 $\pm$ 0.06 & \textbf{0.82 $\pm$ 0.05} \\
 GraphSAGE & 0.72 $\pm$ 0.06 & 0.57 $\pm$ 0.08 & 0.68 $\pm$ 0.09 & 0.61 $\pm$ 0.06 & 0.78 $\pm$ 0.01 \\
 GraphSAGEWithJK & 0.71 $\pm$ 0.09 & 0.56 $\pm$ 0.04 & 0.68 $\pm$ 0.09 & 0.60 $\pm$ 0.07 & 0.77 $\pm$ 0.01  \\
 DiffPool & 0.84 $\pm$ 0.12 & 0.57 $\pm$ 0.03 & 0.69 $\pm$ 0.07 & 0.61 $\pm$ 0.06 & 0.77 $\pm$ 0.02 \\
 GlobalAttentionNet & 0.74 $\pm$ 0.07 & 0.56 $\pm$ 0.05 & 0.67 $\pm$ 0.06 & \textbf{0.63 $\pm$ 0.06} & 0.77 $\pm$ 0.01 \\
 Set2SetNet & 0.73 $\pm$ 0.07 & 0.56 $\pm$ 0.03 & 0.68 $\pm$ 0.10 & 0.62 $\pm$ 0.06 & 0.78 $\pm$ 0.04 \\
 SortPool & 0.75 $\pm$ 0.11 & 0.59 $\pm$ 0.08 & 0.66 $\pm$ 0.09 & 0.60 $\pm$ 0.08 & 0.77 $\pm$ 0.01 \\ \\
 
 Proposed Model & \textbf{0.87 $\pm$ 0.06} & \textbf{0.60 $\pm$ 0.08} & 0.62 $\pm$ 0.10 & 0.62 $\pm$ 0.06 & 0.79 $\pm$ 0.05 \\ [1ex] 
 \hline
\end{tabular}
\caption{For each dataset, the mean classification accuracy plus standard deviation of 10-fold cross validation for each graph classification model are displayed.}
\label{table:mean_accuracy}
\end{table*}

It is important to note that the proposed method was compared against a large number of benchmark methods (17). This makes it challenging for any single method to perform best on all datasets. It is difficult to interpret exactly why one deep learning architecture performs better or worse than another on a particular dataset. However, one limitation of the proposed method that may limit its ability to accurately discriminate is that it only models the distribution of node embeddings and not the position of these nodes in the graph. The recent work by You et al. \cite{pmlr-v97-you19b} suggests position information is important. The DiffPool method which performed best on the BZR\_MD dataset actually uses node position information when performing clustering in the pooling step (this is illustrated in Figure 1 of the original paper by Ying et al. \cite{ying2018hierarchical}). We hypothesize that position information may not be important for some graph classification tasks while being important for others. This may explain why the proposed method does not uniformly outperform all others. It is also worth noting that the proposed method achieved similar performance to the GIN method on the BZR\_MD dataset. In a recent paper by Errica et al. \cite{errica2019fair}, the authors found the GIN method to achieve best results on a number of datasets. Finally, it is interesting to note that the end-to-end deep learning models did not uniformly outperform the feature engineered kernel models. In fact, the best mean performance on the BZR\_MD dataset was achieved by a feature engineered kernel model.

\section{Conclusions and Future Work}
\label{sec:conclusions}
This article proposes a novel kernel-based support vector machine (SVM) for graph classification. Unlike existing kernel-based models, the proposed model is trained in a supervised end-to-end manner whereby the convolutional layers, the kernel function and SVM parameters are jointly optimized. The proposed model outperforms existing deep learning models on a number of datasets which demonstrates the utility of the model.

Despite these positive results, the proposed model is not a suitable candidate solution for all graph classification problems. Like all kernel-based models, the proposed model does not natively scale to large datasets. This is a consequence of the fact that training the model requires computation and storing of the kernel matrix whose size is quadratic in the number of training examples. This limitation may potentially be overcome by performing an approximation of the kennel function \cite{rahimi2008random}. The authors plan to investigate this research direction in future work.

\section*{Acknowledgements}
The author would like to acknowledge the many useful discussions he had with Bertrand Gauthier concerning kernel methods.

\section*{Appendix}
\label{sec:appendix}
We briefly describe the kernel functions corresponding to the feature engineered kernel baseline models considered in the work.

\noindent
\textbf{Graphlet} - This is a kernel in the $\mathcal{R}$-convolution framework which uses substructures based on Graphlets and was proposed by Shervashidze et al. \cite{shervashidze2009efficient}.

\noindent
\textbf{Shortest Path} - This is a kernel in the $\mathcal{R}$-convolution framework which uses substructures based on shortest paths and was proposed by Borgwardt et al. \cite{borgwardt2005shortest}.

\noindent
\textbf{Vertex Histogram} - This is a kernel in the $\mathcal{R}$-convolution which uses substructures based on random walks and was proposed by Sugiyama et al. \cite{sugiyama2015halting}.

\noindent
\textbf{Weisfeiler Lehman} - This is a kernel in the Weisfeiler-Lehman framework which uses the Vertex Histogram Kernel as the base kernel \cite{sugiyama2015halting} and was proposed by Shervashidze et al. \cite{shervashidze2011weisfeiler}.

\noindent
\textbf{Pyramid Match} - This kernel uses unsupervised learning and was proposed by Nikolentzos et al. \cite{nikolentzos2017matching}.

\vspace{.5cm}
We briefly describe the architectures corresponding to the end-to-end deep learning baseline models considered in the work. More specific implementation details can be found at the benchmark section of the PyTorch Geometric website.

\noindent
\textbf{GCN} - This model consists of graph convolutional layers proposed by Kipf et al. \cite{kipf2017semi}, followed by mean pooling, followed by a non-linear layer, followed by a dropout layer, followed by a linear layer, followed by a softmax layer.

\noindent
\textbf{GCNWithJK} - This model is equal to GCN but with the addition of jump or skip connections before mean pooling as proposed by Xu et al. \cite{pmlr-v80-xu18c}.

\noindent
\textbf{GIN} - This model consists of the graph convolutional layers proposed by Xu et al. \cite{xu2018}, followed by mean pooling, followed by a non-linear layer, followed by a dropout layer, followed by a linear layer, followed by a softmax layer. The convolution layer in question has a parameter $\epsilon$ which is learned.

\noindent
\textbf{GIN0} - This model is equal to GIN with the exception that the parameter $\epsilon$ is not learned and instead is set to a value of $0$.

\noindent
\textbf{GINWithJK} - This model is equal to GIN but with the addition of jump or skip connections before mean pooling as proposed by Xu et al. \cite{pmlr-v80-xu18c}.

\noindent
\textbf{GIN0WithJK} - This model is equal to GIN0 but with the addition of jump or skip connections before mean pooling as proposed by Xu et al. \cite{pmlr-v80-xu18c}.

\noindent
\textbf{GraphSAGE} - This model consists of the graph convolutional layers proposed by Hamilton et al. \cite{hamilton2017}, followed by a mean pooling layer, followed by a non-linear layer, followed by a dropout layer, followed by a linear layer, followed by a softmax layer.

\noindent
\textbf{GraphSAGEWithJK} - This model is equal to GraphSAGE but with the addition of jump or skip connections before mean pooling as proposed by Xu et al. \cite{pmlr-v80-xu18c}. 

\noindent
\textbf{DiffPool} - This model consists of the graph convolutional layers proposed by Hamilton et al. \cite{hamilton2017}, followed by the pooling method proposed by Ying et al. \cite{ying2018hierarchical}, followed by a non-linear layer, followed by a dropout layer, followed by a linear layer, followed by a softmax layer.

\noindent
\textbf{GlobalAttentionNet} - This model consists of the graph convolutional layers proposed by Hamilton et al. \cite{hamilton2017}, followed by the pooling layer proposed by Li et al. \cite{li2015gated}, followed by a dropout layer, followed by a non-linear layer, followed by a linear layer, followed by a softmax layer.

\noindent
\textbf{Set2SetNet} - This model consists of the graph convolutional layers proposed by Hamilton et al. \cite{hamilton2017}, followed by the pooling layer proposed by Vinyals et al. \cite{vinyals2015}, followed by a non-linear layer, followed by a dropout layer, followed by a linear layer, followed by a softmax layer.

\noindent
\textbf{SortPool} - This model consists of the graph convolutional layers proposed by Hamilton et al. \cite{hamilton2017}, followed by the pooling layer proposed by Zhang et al. \cite{zhang2018}, followed by a non-linear layer, followed by a dropout layer, followed by linear layer, followed by a softmax layer.


\begin{thebibliography}{10}

\bibitem{borgwardt2005shortest}
Karsten~M Borgwardt and Hans-Peter Kriegel.
\newblock Shortest-path kernels on graphs.
\newblock In {\em IEEE international conference on data mining}, pages 8--pp,
  2005.

\bibitem{chen2019scene}
Vincent Chen, Paroma Varma, Ranjay Krishna, Michael Bernstein, Christopher Re,
  and Li~Fei-Fei.
\newblock Scene graph prediction with limited labels.
\newblock In {\em International Conference on Computer Vision}, 2019.

\bibitem{debnath1991}
Asim~Kumar Debnath, Rosa~L {Lopez de Compadre}, Gargi Debnath, Alan~J
  Shusterman, and Corwin Hansch.
\newblock Structure-activity relationship of mutagenic aromatic and
  heteroaromatic nitro compounds. correlation with molecular orbital energies
  and hydrophobicity.
\newblock {\em Journal of medicinal chemistry}, 34(2):786--797, 1991.

\bibitem{duvenaud2015}
David Duvenaud, Dougal Maclaurin, Jorge Iparraguirre, Rafael Bombarell, Timothy
  Hirzel, Alan Aspuru-Guzik, and Ryan Adams.
\newblock Convolutional networks on graphs for learning molecular fingerprints.
\newblock In {\em Advances in neural information processing systems}, pages
  2224--2232, 2015.

\bibitem{errica2019fair}
Federico Errica, Marco Podda, Davide Bacciu, and Alessio Micheli.
\newblock A fair comparison of graph neural networks for graph classification.
\newblock {\em arXiv preprint arXiv:1912.09893}, 2019.

\bibitem{fey2019}
Matthias Fey, , and Jan~E. Lenssen.
\newblock Fast graph representation learning with {PyTorch Geometric}.
\newblock In {\em ICLR Workshop on Representation Learning on Graphs and
  Manifolds}, 2019.

\bibitem{franceschi2019learning}
Luca Franceschi, Mathias Niepert, Massimiliano Pontil, and Xiao He.
\newblock Learning discrete structures for graph neural networks.
\newblock In {\em International Conference on Machine Learning}, pages
  1972--1982, 2019.

\bibitem{gagarin2018multiple}
Andrei Gagarin and Padraig Corcoran.
\newblock Multiple domination models for placement of electric vehicle charging
  stations in road networks.
\newblock {\em Computers \& Operations Research}, 96:69--79, 2018.

\bibitem{gilmer2017}
Justin Gilmer, Samuel~S Schoenholz, Patrick~F Riley, Oriol Vinyals, and
  George~E Dahl.
\newblock Neural message passing for quantum chemistry.
\newblock In {\em International Conference on Machine Learning}, pages
  1263--1272, 2017.

\bibitem{goodfellow2016deep}
Ian Goodfellow, Yoshua Bengio, and Aaron Courville.
\newblock {\em Deep learning}.
\newblock MIT press, 2016.

\bibitem{hamilton2017}
Will Hamilton, Zhitao Ying, and Jure Leskovec.
\newblock Inductive representation learning on large graphs.
\newblock In {\em Advances in Neural Information Processing Systems}, pages
  1024--1034, 2017.

\bibitem{haussler1999convolution}
David Haussler.
\newblock Convolution kernels on discrete structures.
\newblock Technical report, 1999.

\bibitem{he2015delving}
Kaiming He, Xiangyu Zhang, Shaoqing Ren, and Jian Sun.
\newblock Delving deep into rectifiers: Surpassing human-level performance on
  imagenet classification.
\newblock In {\em Proceedings of the IEEE international conference on computer
  vision}, pages 1026--1034, 2015.

\bibitem{pmlr-v80-ivanov18a}
Sergey Ivanov and Evgeny Burnaev.
\newblock Anonymous walk embeddings.
\newblock In {\em International Conference on Machine Learning}, volume~80,
  pages 2191--2200, Stockholmsmassan, Stockholm Sweden, 10--15 Jul 2018.

\bibitem{kearnes2016}
Steven Kearnes, Kevin McCloskey, Marc Berndl, Vijay Pande, and Patrick Riley.
\newblock Molecular graph convolutions: moving beyond fingerprints.
\newblock {\em Journal of computer-aided molecular design}, 30(8):595--608,
  2016.

\bibitem{KKMMN2016}
Kristian Kersting, Nils~M. Kriege, Christopher Morris, Petra Mutzel, and Marion
  Neumann.
\newblock Benchmark data sets for graph kernels, 2016.

\bibitem{kingma2014adam}
Diederik~P Kingma and Jimmy Ba.
\newblock Adam: A method for stochastic optimization.
\newblock {\em arXiv preprint arXiv:1412.6980}, 2014.

\bibitem{kipf2017semi}
Thomas~N. Kipf and Max Welling.
\newblock Semi-supervised classification with graph convolutional networks.
\newblock In {\em International Conference on Learning Representations}, 2017.

\bibitem{kondor2016multiscale}
Risi Kondor and Horace Pan.
\newblock The multiscale laplacian graph kernel.
\newblock In {\em Advances in Neural Information Processing Systems}, pages
  2990--2998, 2016.

\bibitem{kriege2016valid}
Nils~M Kriege, Pierre-Louis Giscard, and Richard Wilson.
\newblock On valid optimal assignment kernels and applications to graph
  classification.
\newblock In {\em Advances in Neural Information Processing Systems}, pages
  1623--1631, 2016.

\bibitem{krishna2017visual}
Ranjay Krishna, Yuke Zhu, Oliver Groth, Justin Johnson, Kenji Hata, Joshua
  Kravitz, Stephanie Chen, Yannis Kalantidis, Li-Jia Li, David~A Shamma, et~al.
\newblock Visual genome: Connecting language and vision using crowdsourced
  dense image annotations.
\newblock {\em International Journal of Computer Vision}, 123(1):32--73, 2017.

\bibitem{Landrieu_2018_CVPR}
Loic Landrieu and Martin Simonovsky.
\newblock Large-scale point cloud semantic segmentation with superpoint graphs.
\newblock In {\em The IEEE Conference on Computer Vision and Pattern
  Recognition}, June 2018.

\bibitem{li2015gated}
Yujia Li, Daniel Tarlow, Marc Brockschmidt, and Richard Zemel.
\newblock Gated graph sequence neural networks.
\newblock In {\em International Conference on Learning Representations}, 2016.

\bibitem{luzhnica2019graph}
Enxhell Luzhnica, Ben Day, and Pietro Li{\`o}.
\newblock On graph classification networks, datasets and baselines.
\newblock In {\em ICML Workshop on Learning and Reasoning with Graph-Structured
  Representations}, California, United States, 2019.

\bibitem{nikolentzos2018kernel}
Giannis Nikolentzos, Polykarpos Meladianos, Antoine Jean-Pierre Tixier,
  Konstantinos Skianis, and Michalis Vazirgiannis.
\newblock Kernel graph convolutional neural networks.
\newblock In {\em International Conference on Artificial Neural Networks},
  pages 22--32. Springer, 2018.

\bibitem{nikolentzos2017matching}
Giannis Nikolentzos, Polykarpos Meladianos, and Michalis Vazirgiannis.
\newblock Matching node embeddings for graph similarity.
\newblock In {\em AAAI Conference on Artificial Intelligence}, 2017.

\bibitem{nocedal2006numerical}
Jorge Nocedal and Stephen Wright.
\newblock {\em Numerical optimization}.
\newblock Springer Science \& Business Media, 2006.

\bibitem{paulsen2016introduction}
Vern~I Paulsen and Mrinal Raghupathi.
\newblock {\em An introduction to the theory of reproducing kernel Hilbert
  spaces}, volume 152.
\newblock Cambridge University Press, 2016.

\bibitem{rahimi2008random}
Ali Rahimi and Benjamin Recht.
\newblock Random features for large-scale kernel machines.
\newblock In {\em Advances in neural information processing systems}, pages
  1177--1184, 2008.

\bibitem{rieck2019persistent}
Bastian Rieck, Christian Bock, and Karsten Borgwardt.
\newblock A persistent weisfeiler-lehman procedure for graph classification.
\newblock In {\em International Conference on Machine Learning}, pages
  5448--5458, 2019.

\bibitem{scholkopf2002learning}
Bernhard Sch{\"o}lkopf, Alexander~J Smola, Francis Bach, et~al.
\newblock {\em Learning with kernels: support vector machines, regularization,
  optimization, and beyond}.
\newblock MIT press, 2002.

\bibitem{shchur2018pitfalls}
Oleksandr Shchur, Maximilian Mumme, Aleksandar Bojchevski, and Stephan
  G{\"u}nnemann.
\newblock Pitfalls of graph neural network evaluation.
\newblock {\em arXiv preprint arXiv:1811.05868}, 2018.

\bibitem{shervashidze2011weisfeiler}
Nino Shervashidze, Pascal Schweitzer, Erik Jan~van Leeuwen, Kurt Mehlhorn, and
  Karsten~M Borgwardt.
\newblock Weisfeiler-lehman graph kernels.
\newblock {\em Journal of Machine Learning Research}, 12(Sep):2539--2561, 2011.

\bibitem{shervashidze2009efficient}
Nino Shervashidze, SVN Vishwanathan, Tobias Petri, Kurt Mehlhorn, and Karsten
  Borgwardt.
\newblock Efficient graphlet kernels for large graph comparison.
\newblock In {\em Artificial Intelligence and Statistics}, pages 488--495,
  2009.

\bibitem{JMLR:v21:18-370}
Giannis Siglidis, Giannis Nikolentzos, Stratis Limnios, Christos Giatsidis,
  Konstantinos Skianis, and Michalis Vazirgiannis.
\newblock Grakel: A graph kernel library in python.
\newblock {\em Journal of Machine Learning Research}, 21(54):1--5, 2020.

\bibitem{sugiyama2015halting}
Mahito Sugiyama and Karsten Borgwardt.
\newblock Halting in random walk kernels.
\newblock In {\em Advances in neural information processing systems}, pages
  1639--1647, 2015.

\bibitem{sutherland2003spline}
Jeffrey~J Sutherland, Lee~A O'brien, and Donald~F Weaver.
\newblock Spline-fitting with a genetic algorithm: A method for developing
  classification structure- activity relationships.
\newblock {\em Journal of chemical information and computer sciences},
  43(6):1906--1915, 2003.

\bibitem{vinyals2015}
Oriol Vinyals, Samy Bengio, and Manjunath Kudlur.
\newblock Order matters: Sequence to sequence for sets.
\newblock In {\em International Conference on Learning Representations}, San
  Juan, Puerto Rico, 2016.

\bibitem{wu2019scalable}
Lingfei Wu, Ian En-Hsu Yen, Zhen Zhang, Kun Xu, Liang Zhao, Xi~Peng, Yinglong
  Xia, and Charu Aggarwal.
\newblock Scalable global alignment graph kernel using random features: From
  node embedding to graph embedding.
\newblock In {\em Proceedings of the 25th ACM SIGKDD International Conference
  on Knowledge Discovery \& Data Mining}, pages 1418--1428. ACM, 2019.

\bibitem{wu2019comprehensive}
Zonghan Wu, Shirui Pan, Fengwen Chen, Guodong Long, Chengqi Zhang, and Philip~S
  Yu.
\newblock A comprehensive survey on graph neural networks.
\newblock {\em arXiv preprint arXiv:1901.00596}, 2019.

\bibitem{xu2017scenegraph}
Danfei Xu, Yuke Zhu, Christopher Choy, and Li~Fei-Fei.
\newblock Scene graph generation by iterative message passing.
\newblock In {\em The IEEE Conference on Computer Vision and Pattern
  Recognition}, 2017.

\bibitem{xu2018}
Keyulu Xu, Weihua Hu, Jure Leskovec, and Stefanie Jegelka.
\newblock How powerful are graph neural networks?
\newblock In {\em International Conference on Learning Representations}, 2019.

\bibitem{pmlr-v80-xu18c}
Keyulu Xu, Chengtao Li, Yonglong Tian, Tomohiro Sonobe, Ken-ichi Kawarabayashi,
  and Stefanie Jegelka.
\newblock Representation learning on graphs with jumping knowledge networks.
\newblock In {\em International Conference on Machine Learning}, volume~80,
  pages 5453--5462, Stockholmsmassan, Stockholm Sweden, 10--15 Jul 2018.

\bibitem{yan2018spatial}
Sijie Yan, Yuanjun Xiong, and Dahua Lin.
\newblock Spatial temporal graph convolutional networks for skeleton-based
  action recognition.
\newblock In {\em AAAI Conference on Artificial Intelligence}, 2018.

\bibitem{yanardag2015deep}
Pinar Yanardag and SVN Vishwanathan.
\newblock Deep graph kernels.
\newblock In {\em ACM SIGKDD International Conference on Knowledge Discovery
  and Data Mining}, pages 1365--1374, 2015.

\bibitem{ying2018hierarchical}
Zhitao Ying, Jiaxuan You, Christopher Morris, Xiang Ren, Will Hamilton, and
  Jure Leskovec.
\newblock Hierarchical graph representation learning with differentiable
  pooling.
\newblock In {\em Advances in Neural Information Processing Systems}, pages
  4800--4810, 2018.

\bibitem{pmlr-v97-you19b}
Jiaxuan You, Rex Ying, and Jure Leskovec.
\newblock Position-aware graph neural networks.
\newblock In Kamalika Chaudhuri and Ruslan Salakhutdinov, editors, {\em
  Proceedings of the 36th International Conference on Machine Learning},
  volume~97 of {\em Proceedings of Machine Learning Research}, pages
  7134--7143, Long Beach, California, USA, June 2019.

\bibitem{you2018graphrnn}
Jiaxuan You, Rex Ying, Xiang Ren, William Hamilton, and Jure Leskovec.
\newblock Graphrnn: Generating realistic graphs with deep auto-regressive
  models.
\newblock In {\em International Conference on Machine Learning}, pages
  5694--5703, 2018.

\bibitem{zhang2018}
Muhan Zhang, Zhicheng Cui, Marion Neumann, and Yixin Chen.
\newblock An end-to-end deep learning architecture for graph classification.
\newblock In {\em AAAI Conference on Artificial Intelligence}, 2018.

\bibitem{zhang2018retgk}
Zhen Zhang, Mianzhi Wang, Yijian Xiang, Yan Huang, and Arye Nehorai.
\newblock Retgk: Graph kernels based on return probabilities of random walks.
\newblock In {\em Advances in Neural Information Processing Systems}, pages
  3964--3974, 2018.

\bibitem{zhang2018deep}
Ziwei Zhang, Peng Cui, and Wenwu Zhu.
\newblock Deep learning on graphs: A survey.
\newblock {\em arXiv preprint arXiv:1812.04202}, 2018.

\end{thebibliography}
\end{document}